\documentclass[letterpaper, 10 pt, conference]{ieeeconf}  
\usepackage{graphicx}
\usepackage{adjustbox}
\usepackage{multirow}
\usepackage{amsmath}
\usepackage{amssymb}
\newtheorem{proposition}{Proposition}
\newtheorem{theorem}{Theorem}
\newtheorem{corollary}{Corollary}
\newtheorem{definition}{Definition}
\usepackage[ruled,vlined,linesnumbered]{algorithm2e}
\usepackage{cite}
\usepackage{subcaption}
\usepackage{xcolor}
\usepackage{makecell}
\usepackage{amssymb} 
\usepackage{pifont}  
\usepackage[normalem]{ulem}
\usepackage{xcolor}
\usepackage{hyperref}

\hypersetup{
    colorlinks=true,
    urlcolor=magenta
}

\newcounter{assumption}

\newcommand{\cmark}{{\color{green}\checkmark}}
\newcommand{\xmark}{{\color{red}\ding{55}}}

\IEEEoverridecommandlockouts                              

\title{\LARGE \bf
    PC-Diffuser: Path-Consistent Capsule CBF Safety Filtering for Diffusion-Based Trajectory Planner
}

\author{
    Eugene Ku$^{1}$ and Yiwei Lyu$^{1}$ 
    \thanks{$^{1}$Eugene Ku and Yiwei Lyu are with the Department of Computer Science and Engineering, Texas A\&M University, College Station, TX, USA.
        {\tt\small yjean234@tamu.edu; yiweilyu@tamu.edu}}%
}

\begin{document}

\maketitle
\thispagestyle{empty}
\pagestyle{empty}


\begin{abstract}
Autonomous driving in complex traffic requires planners that generalize beyond hand-crafted rules, motivating data-driven approaches that learn behavior from expert demonstrations. Diffusion-based trajectory planners have recently shown strong closed-loop performance by iteratively denoising a full-horizon plan, but they remain difficult to certify and can fail catastrophically in rare or out-of-distribution scenarios. To address this challenge, we present PC-Diffuser, a safety augmentation framework that embeds a certifiable, path-consistent barrier-function structure directly into the denoising loop of diffusion planning. The key idea is to make safety an intrinsic part of trajectory generation rather than a post-hoc fix: we enforce forward invariance along the rollout while preserving the diffusion model’s intended path geometry. Specifically, PC-Diffuser (i) evaluates collision risk using a capsule-distance barrier function that better reflects vehicle geometry and reduces unnecessary conservativeness, (ii) converts denoised waypoints into dynamically feasible motion under a kinematic bicycle model, and (iii) applies a path-consistent safety filter that eliminates residual constraint violations without geometric distortion, so the corrected plan remains close to the learned distribution. By injecting these safety-consistent corrections at every denoising step and feeding the refined trajectory back into the diffusion process, PC-Diffuser enables iterative, context-aware safeguarding instead of post-hoc repair. On the nuPlan closed-loop benchmark, PC-Diffuser reduces collision rate from 100\% to 10.29\% on the all-collision challenge set (a collision subset of Val14), outperforming representative safety augmentation baselines, while showing improvement in the composite driving score on both Val14 and Test14-hard relative to the base diffusion planner, demonstrating its ability to safeguard without compromising driving performance. Code and videos are available at: \url{https://eugene29.github.io/PC-Diffuser_website/}.

\end{abstract}


\section{INTRODUCTION}




Motion planning for autonomous driving must handle a wide spectrum of traffic interactions that are difficult to fully anticipate with hand-crafted rules. This has motivated data-driven planners that learn driving behavior from expert demonstrations, enabling generalization across diverse real-world scenarios~\cite{caesar2021nuplan}. Among these approaches, diffusion-based planners have recently emerged as a strong paradigm for long-horizon planning~\cite{janner2022diffuser, ho2020ddpm}. By starting from a noisy trajectory and iteratively denoising it into a coherent plan, diffusion planners refine the entire horizon jointly, which helps avoid the compounding errors of autoregressive action prediction and yields globally consistent behavior over long rollouts~\cite{janner2022diffuser, zheng2025diffusionplanner, tan2025flowplanner}. 

Despite these advantages, diffusion planners offer limited reliability guarantees. Their learned score function encourages trajectories that resemble the training distribution, but provides no formal mechanism to prevent unsafe outcomes when the scene departs from the data manifold. In rare or out-of-distribution traffic configurations, the generated plan can be physically plausible yet unsafe, including collision-inducing behaviors. This gap between strong average-case performance and the absence of worst-case guarantees is a central barrier to safety-critical deployment.


\begin{figure}[h]
    \centering
\includegraphics[width=\linewidth,trim=.7cm 3.5cm 6.cm 3.5cm,clip]{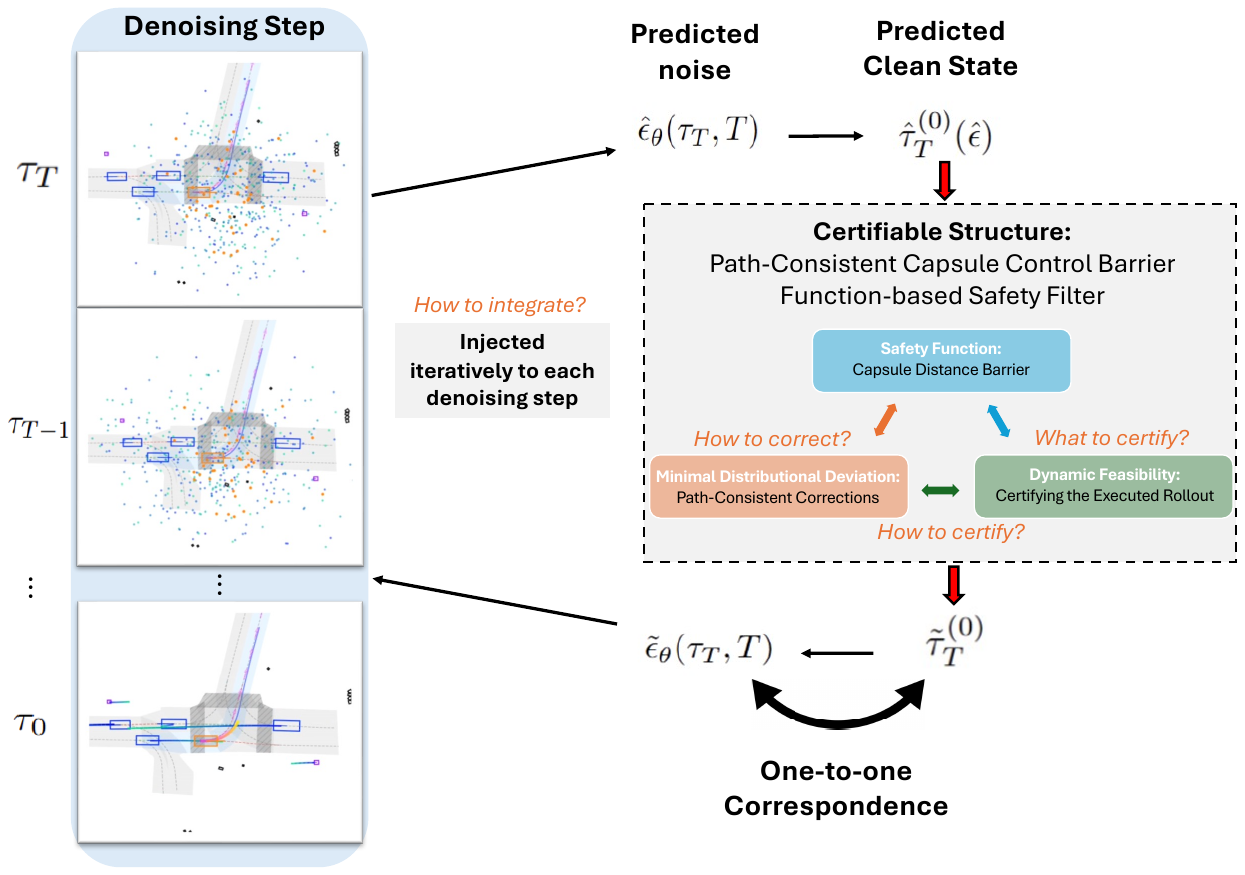}
    \caption{Overview of the proposed PC-Diffuser safety augmentation framework.}
    \label{fig:PC-Diffuser}
\end{figure}

A natural response is to combine diffusion planning with formal safety enforcement mechanisms such as control barrier functions (CBFs). However, existing integrations often miss at least one of three requirements that matter in practice.  First, safety should be certified on the trajectory that will be executed under the vehicle dynamics (i.e., the rollout induced by tracking/control), not merely on intermediate diffusion iterates or raw waypoint sequences. Second, that certification must be dynamics-consistent: enforcing purely geometric constraints on waypoints, or using an inconsistent motion model, can produce “safe” plans that are not physically realizable. Third, the correction should be minimally invasive, preserving the planner’s intended path as much as possible to avoid large distributional shifts that degrade driving quality, inducing overly conservative or even unsafe behavior. Achieving all three simultaneously is nontrivial: enforcing CBF constraints aggressively can distort the learned behavior, while enforcing them weakly (or at the wrong time index) can fail to certify the executed trajectory.

To bridge this gap, we propose \textbf{PC-Diffuser}, a framework that injects a certifiable, path-consistent structure into the denoising loop so that safety is enforced during generation rather than repaired afterward. In doing so, we aim to answer three fundamental questions when it comes to improving safety for diffusion-based planning: (1) \emph{Which is the right object to certify?} (i.e., the executed rollout induced by tracking/control rather than intermediate diffusion iterates or waypoint sequences); (2) \emph{How can we make certification dynamics-consistent?} so that safety guarantees align with what the vehicle can physically execute; and (3) \emph{How should the safety correction change the plan?} such that violations are resolved with minimal disruption to the planner’s intended path and without introducing large distributional shifts that degrade driving quality.


The \textbf{main contributions} of this paper are: (i) We introduce a certifiable, path-consistent barrier-function structure that jointly supports rollout-time safety, dynamic feasibility, and minimal deviation from the learned diffusion behavior, enabling safety enforcement that is both physically meaningful and minimally invasive. (ii) We integrate this structure directly into the denoising loop, allowing iterative, context-aware safeguarding during generation and thereby reducing reliance on one-shot post-processing that can be brittle in rare or out-of-distribution scenarios. (iii) We evaluate on the nuPlan closed-loop benchmark~\cite{caesar2021nuplan} and show that PC-Diffuser eliminates the majority of catastrophic failures on the all-collision challenge set, driving the collision rate down from 100\% to 10.29\%, outperforming popular baseline methods. This is done without compromising driving performance, demonstrated by an improvement in nuplan's composite score on standard Val14 and Test14-hard splits relative to the vanilla baseline diffusion planner.

\section{Related Works}

%





A growing body of work has sought to improve the safety of diffusion-based planners, spanning guidance at the score level, constrained denoising, and post-processing safety layers.

\noindent\textbf{Score-guidance for safety.} DiffusionPlanner~\cite{zheng2025diffusionplanner} steers the denoising process toward safer behaviors by backpropagating gradients from a hand-crafted reward classifier into the diffusion score. While effective in biasing samples, gradient-based guidance does not provide a formal safety certificate and can introduce artifacts that are dynamically inconsistent, since the guided trajectory is not necessarily generated through a dynamics-respecting rollout. Additionally, hand-crafted reward classifiers often do not provide meaningful gradients and can disrupt the model's learned driving behavior. 

\noindent \textbf{Constrained denoising via CBFs.}  SafeDiffuser~\cite{xiao2023safediffuser} and SafeFlow~\cite{dai2025safeflow, huang2026sadflower} introduce {Constrained denoising} methods through embedding Control Barrier Function constraints directly into the denoising process, enforcing forward invariance across the diffusion index rather than over rollout time. While this guarantees that consecutive denoising iterates remain in a CBF-safe set, the diffusion index is not rollout time: the final executed trajectory carries no formal safety certificate, and dynamic feasibility is left unaddressed.


\noindent \textbf{Optimization-based safety layers.} A complementary line of work enforces safety by filtering controls or trajectories through constrained optimization. CBF-QP~\cite{ames2017cbfqp} applies a reactive safety filter at each control step, which can ignore the planner’s long-horizon intent and induce conservative, myopic deviations from the learned behavior. MPC-CBF~\cite{zeng2021mpccbf} incorporates lookahead by imposing CBF constraints within a receding-horizon formulation, but the resulting problem is often non-convex and challenging to solve reliably in real time, and the corrected solutions can drift away from the learned distribution.



\noindent\textbf{Validation and fallback. } PACS~\cite{romer2025pacs} and RAIL~\cite{jung2025rail} validate a learned plan via reachability analysis and switch to a backup policy when collision is predicted. This preserves the original plan when already safe, but the fallback behavior, often emergency braking, can itself be hazardous or overly disruptive in dense, interactive traffic.


In contrast to the above approaches, our goal is to couple diffusion generation with certifiable safety in a way that targets the executed rollout. Specifically, we enforce CBF forward invariance over rollout time instead of diffusion time, design the correction to be minimally disruptive so the plan remains close to the intended trajectory and learned distribution, and jointly search for a safe trajectory with the diffusion planner rather than relying on a rule-defined fallback policy.





\section{Preliminaries}
\subsection{Vehicle Dynamics}
We model the ego vehicle with a kinematic bicycle model with state $\mathbf{x} = (x, y, \theta, \delta, v) \in \mathbb{R}^5$ (position, heading, steering angle, speed) and control $\mathbf{u} = (\dot\delta, a)$ (steering rate, acceleration). The dynamics are $\dot{\mathbf{x}} = f(\mathbf{x}) + g(\mathbf{x})\,\mathbf{u}$ and control affine in $(\mathbf{x},\mathbf{u})$, with
\begin{equation}
    f(\mathbf{x}) = \begin{bmatrix} v\cos\theta \\ v\sin\theta \\ \frac{v\tan\delta}{L} \\ 0 \\ 0 \end{bmatrix}, \quad
    g(\mathbf{x}) = \begin{bmatrix} 0 & 0 \\ 0 & 0 \\ 0 & 0 \\ 1 & 0 \\ 0 & 1 \end{bmatrix},
    \label{eq:bicycle}
\end{equation}
where $L$ is the wheelbase.

\subsection{Control Barrier Functions}
For a control-affine system, a continuously differentiable function $h: \mathbb{R}^n \to \mathbb{R}$ is a \emph{Control Barrier Function} (CBF)~\cite{ames2019cbf} for the safe set $\mathcal{S} = \{\mathbf{x} : h(\mathbf{x}) \geq 0\}$ if there exists an extended class-$\mathcal{K}_\infty$ function $\alpha$ such that
\begin{equation}
    \sup_{u \in \mathcal{U}}\!\left[\nabla h(\mathbf{x})^\top\!\left(f(\mathbf{x}) + g(\mathbf{x})u\right)\right] \geq -\alpha\!\left(h(\mathbf{x})\right),
    \label{eq:cbf}
\end{equation}
for all $\mathbf{x} \in \mathcal{S}$.
Any Lipschitz controller satisfying~\eqref{eq:cbf} renders $\mathcal{S}$ forward invariant~\cite{ames2019cbf}.
The minimally invasive safe controller is obtained via the \emph{CBF-QP}:
\begin{equation}
\begin{aligned}
    u^* = \arg\min_{u \in \mathcal{U}}\;& \|u - u_{\mathrm{nom}}\|^2 \\
    \mathrm{s.t.}\quad& \nabla h(\mathbf{x})^\top\!\left(f(\mathbf{x}) + g(\mathbf{x})\,u\right) \geq -\alpha\!\left(h(\mathbf{x})\right).
\end{aligned}
\label{eq:cbf_qp}
\end{equation}

\subsection{Diffusion-Based Planning}
We adopt the denoising diffusion probabilistic model (DDPM)~\cite{ho2020ddpm} for trajectory planning and generate an ego-state trajectory over a rollout horizon $K$. Let $\boldsymbol{\tau} = (\mathbf{x}_1,\ldots,\mathbf{x}_K)\in\mathbb{R}^{K\times 4}$ ($x, y, \cos\theta, \sin\theta$) denote the trajectory, where each $\mathbf{x}_k$ is the ego state at rollout time $k$. Diffusion planning generates $\boldsymbol{\tau}$ by running a learned reverse-time denoising process for $T$ steps, starting from isotropic Gaussian noise $\boldsymbol{\tau}_T \sim \mathcal{N}(\mathbf{0},\mathbf{I})$.

At each diffusion step $t$, the denoising network $\epsilon_\theta(\boldsymbol{\tau}_t,t)$ predicts the noise component in the current noisy sample $\boldsymbol{\tau}_t$ which can be used to form an estimate of the underlying clean trajectory:
\begin{equation}
    \hat{\boldsymbol{\tau}}_0^{(t)} = \frac{1}{\sqrt{\bar{\alpha}_t}}
    \left(\boldsymbol{\tau}_t - \sqrt{1-\bar{\alpha}_t}\;\epsilon_\theta(\boldsymbol{\tau}_t,t)\right),
    \label{eq:clean_estimate}
\end{equation}
where $\bar{\alpha}_t$ denotes the cumulative noise schedule. The next iterate $\boldsymbol{\tau}_{t-1}$ is then sampled from the DDPM reverse transition:
\begin{equation}
    \boldsymbol{\tau}_{t-1} = \sqrt{\bar{\alpha}_{t-1}}\;\hat{\boldsymbol{\tau}}_0^{(t)}
    + \sqrt{1-\bar{\alpha}_{t-1}}\;\hat{\boldsymbol{\epsilon}}^{(t)} + \sigma_t \mathbf{z},
    \label{eq:ddpm_step}
\end{equation}
where $\hat{\boldsymbol{\epsilon}}^{(t)} = \epsilon_\theta(\boldsymbol{\tau}_t,t)$, $\mathbf{z}\sim\mathcal{N}(\mathbf{0},\mathbf{I})$, and $\sigma_t$ controls the injected sampling noise. Repeating~\eqref{eq:ddpm_step} from $t=T$ down to $t=0$ yields a final denoised trajectory, which serves as the planner's output. Throughout the paper, we use $k$ to index rollout time, $t$ to index diffusion time, and $j$ to index agents.


\section{Methodology}




Our framework is guided by two motivating questions:
\begin{enumerate}
    \item \emph{What} certifiable structure should be incorporated into a diffusion planner?
    \item \emph{How} should the certifiable structure be integrated into the diffusion planner?
\end{enumerate}

A key design choice underlying both questions is \emph{what object we certify}. 
To address Q1, we identify three properties a certifiable structure must satisfy and introduce a mechanism for each: a capsule-distance control barrier function for \emph{safety}, a path-tracking controller to map waypoints to dynamically feasible controls for \emph{dynamic feasibility}, and a path-consistent correction for \emph{minimal distributional deviation}. 
To address Q2, we integrate this structure into \emph{every} denoising step by correcting the predicted clean trajectory estimate and re-injecting it into the diffusion process, enabling the planner to co-adapt to safety corrections rather than applying a one-shot post-hoc filter. We detail each component below.


\subsection{Safety: Capsule Distance Barrier Function}

We achieve the first property, \emph{safety}, using the control barrier function (CBF) framework (Section~III-B), which enforces forward invariance of a collision-free set through inequality constraints on the time derivative of a barrier function $h$.
A common choice defines $h$ using Euclidean distance between vehicle center points, but this inflates the effective collision boundary and can be overly conservative in tight geometries (e.g., intersections and narrow lanes).
Instead, we represent each vehicle by its longitudinal axis (a line segment connecting the centers of the front and rear ends) and define the barrier using the \emph{capsule distance}, i.e., the minimum distance between the two segments (Fig.~\ref{fig:capsule_distance}).

\begin{figure}[h]
    \centering
    \includegraphics[width=.7\columnwidth]{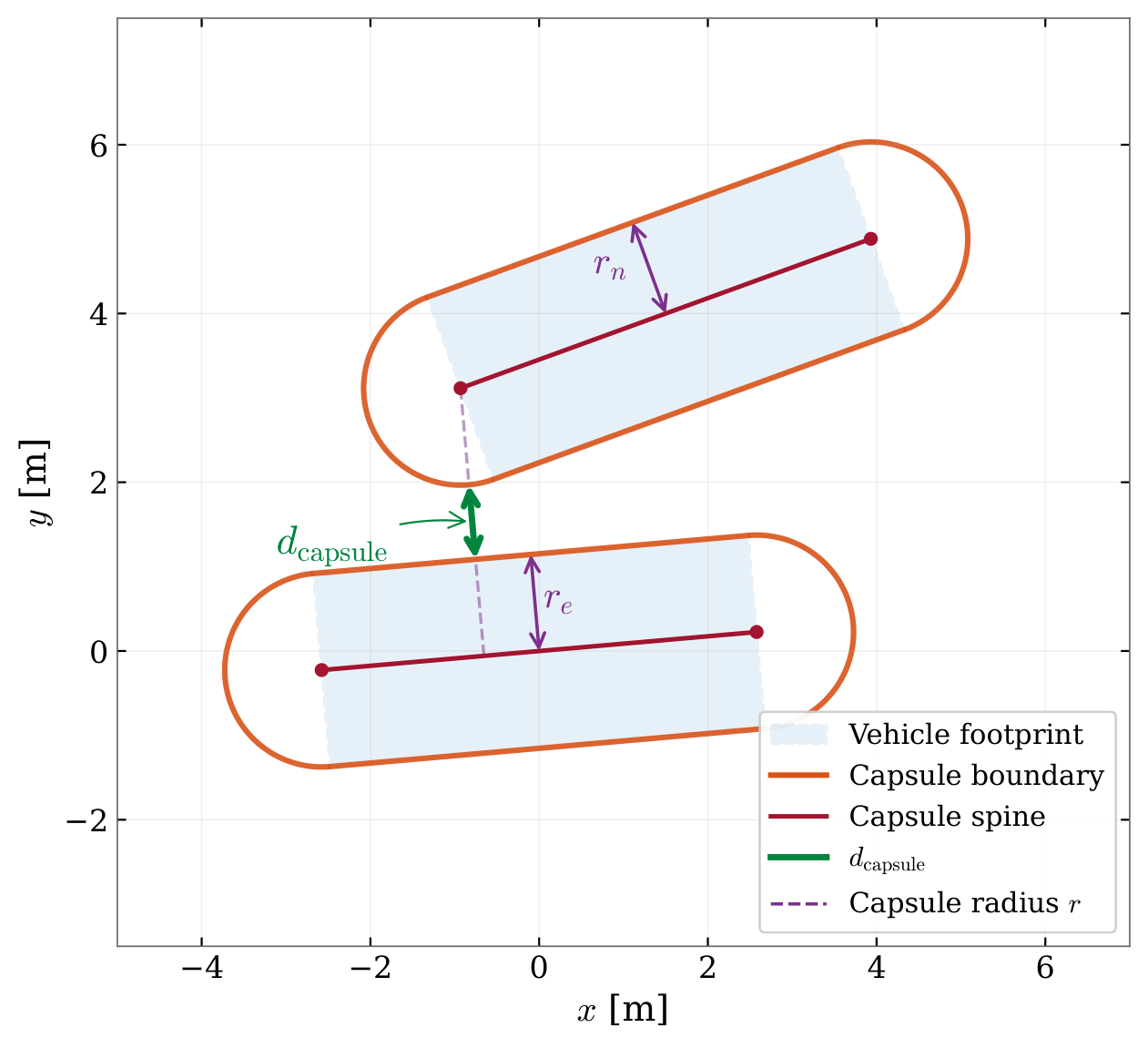}
    \caption{Capsule distance diagram. Each vehicle is represented by its longitudinal axis (line segment). Then the capsule distance is the minimum distance between the two segments minus $r_n + r_e$, the sum of half widths of the two vehicles.}
    \label{fig:capsule_distance}
\end{figure}

Specifically, we represent each vehicle's longitudinal axis as a line segment connecting its rear-end center $P$ to its front-end center $Q$. Given two such segments parameterized by $S_1(s) = P_1 + s(Q_1 - P_1)$ and $S_2(r) = P_2 + r(Q_2 - P_2)$ for $s, r \in [0,1]$, the capsule distance is
\begin{equation}
    d_{\text{cap}}(S_1, S_2) = \min_{s,\, r \,\in\, [0,1]} \| S_1(s) - S_2(r) \|.
    \label{eq:capsule_distance}
\end{equation}
We define the barrier function for ego and neighbor $j$ as
\begin{equation}
    h^j(\mathbf{x}) = d_{\text{cap}}\!\left(S_{\text{ego}}(\mathbf{x}),\; S_j(\mathbf{x}^j)\right) - d_{\text{safe}},
    \label{eq:capsule_cbf}
\end{equation}
where $d_{\text{safe}} > 0$ is a safety margin.

\begin{proposition}[Smoothness of the Capsule Barrier]
\label{prop:smooth}
For any ego state $\mathbf{x}=(x,y,\theta,\delta,v)$ and neighbor state $\mathbf{x}^j$ with $d_{\mathrm{cap}}(S_{\mathrm{ego}}(\mathbf{x}), S_j(\mathbf{x}^j))>0$, the capsule barrier
$h^j(\mathbf{x}) = d_{\mathrm{cap}}(S_{\mathrm{ego}}(\mathbf{x}), S_j(\mathbf{x}^j)) - d_{\mathrm{safe}}$
is continuously differentiable with respect to~$\mathbf{x}$ whenever the closest-point pair attaining $d_{\mathrm{cap}}$ is unique.
\end{proposition}

\begin{proof}
Define the squared capsule distance
\begin{align}
  D(\mathbf{x}) \;&=\; \min_{(s,r)\in[0,1]^2}
  \bigl\|S_1(s,\mathbf{x})-S_2(r,\mathbf{x}^j)\bigr\|^2 \nonumber \\
  &\triangleq\; \min_{(s,r)\in[0,1]^2} \phi(\mathbf{x},s,r).
\end{align}

For each fixed $(s,r)$, $\phi$ is $C^\infty$ in~$\mathbf{x}$, since $S_1$ depends on~$\mathbf{x}$ through $\cos\theta$ and $\sin\theta$.
The constraint set $[0,1]^2$ is compact, so the minimum is attained.
By Danskin's theorem~\cite{danskin1966}, when the minimizer $(s^*\!,r^*)$ is unique,
\begin{align}
  \nabla_{\!\mathbf{x}} D &=\; \nabla_{\!\mathbf{x}} \phi(\mathbf{x},s^*\!,r^*) \\
  \;&=\; 2\bigl(S_1(s^*\!,\mathbf{x})-S_2(r^*\!,\mathbf{x}^j)\bigr)^{\!\top}
    \frac{\partial S_1}{\partial \mathbf{x}}\!(s^*\!;\mathbf{x}).
\end{align}
Uniqueness holds whenever the segments are not parallel, which is the generic case.
The chain rule gives $\nabla_{\!\mathbf{x}} d_{\mathrm{cap}} = \hat{n}^{\top}\! \frac{\partial S_1}{\partial\mathbf{x}}$, where $\hat{n}$ is the unit vector from the closest point on~$S_2$ to the closest point on~$S_1$.
\end{proof}

This ensures that $\dot{h}^j$ is well-defined along the rollout dynamics, so we can enforce CBF-based safety constraints once we specify which control input we are allowed to adjust in Section~\ref{sec:pc_cbf}.
\subsection{Dynamic Feasibility: Certifying the Executed Rollout}

Diffusion planners produce waypoint trajectories, not the control inputs $(a,\delta)$ needed to certify safety under vehicle dynamics. A natural alternative is to generate action sequences directly~\cite{mizuta2024cobl, wang2025alpamayo} instead of target waypoints, but this can reintroduce compounding errors and training instability~\cite{mizuta2024cobl}, undermining the advantage of trajectory-level diffusion. We instead retain waypoint generation and introduce an explicit \emph{interface} from waypoints to a dynamically feasible rollout.

Given the predicted clean trajectory $\hat{\boldsymbol{\tau}}_0^{(t)} = (\hat{\mathbf{x}}_1, \ldots, \hat{\mathbf{x}}_K)$, we track it sequentially with a linearized LQR controller~\cite{rajamani2011vehicle} to produce a nominal control $\mathbf{u}_{\mathrm{nom},k} = (a_{\mathrm{nom},k},\, \delta_{\mathrm{nom},k})$ at each rollout step $k$. This nominal control respects the kinematic bicycle model~\eqref{eq:bicycle} by construction and induces a dynamically feasible nominal rollout. Importantly, it also provides the control inputs that the CBF-QP~\eqref{eq:cbf_qp} requires, unifying trajectory-space planning with action-space safety enforcement within a single framework.

\subsection{Minimal Distributional Deviation: Path-Consistent Corrections}

While a generic safety filter may perturb both acceleration and steering, such corrections can alter the planned path geometry and push the vehicle away from the diffusion model’s learned behavior. This can degrade driving quality and may introduce unsafe side effects (e.g., drifting toward adjacent lanes). We therefore require safety corrections to be \emph{path-consistent}: they should primarily adjust \emph{how fast} the vehicle traverses the planned path rather than \emph{where} it goes.

We implement path-consistency by fixing steering to the tracked nominal value $\delta_k=\delta_{\mathrm{nom},k}$ and allowing the safety filter to modify only the longitudinal channel. With steering held fixed, the safety-corrected control is obtained by solving
\begin{equation}
\begin{aligned}
    a_k^* = \arg\min_{a_k \in \mathcal{U}_a}\;& \|a_k - a_{\mathrm{nom},k}\|^2 \\
    \mathrm{s.t.}\quad& \nabla h^j(\mathbf{x}_k)^\top\!\bigl(f(\mathbf{x}_k) + g(\mathbf{x}_k)\,[a_k,\, \delta_{\mathrm{nom},k}]^\top\bigr) \\
    &\qquad \geq -\alpha\!\bigl(h^j(\mathbf{x}_k)\bigr), \quad \forall\, j \in \mathcal{R},
\end{aligned}
\label{eq:pc_cbf_qp}
\end{equation}
where $\mathcal{U}_a$ denotes admissible accelerations and $\mathcal{R}$ is the set of safety-critical agents (defined in Section~\ref{sec:practical}). This restriction preserves the spatial geometry of the planned path by preventing lateral deviations, while still providing sufficient authority to avoid collisions through speed modulation. In the next subsection, we show how this constrained correction can be implemented efficiently using an equivalent velocity-level formulation.

\subsection{Path-Consistent Capsule CBF (PC-CBF) Safety Filter}
\label{sec:pc_cbf}
We now compose the three components introduced above into a unified \emph{PC-CBF safety filter}. PC-CBF uses the capsule barrier to enforce safety on the \emph{executed} rollout induced by the bicycle dynamics, while preserving the planned path geometry by restricting safety corrections primarily to the longitudinal direction. Concretely, PC-CBF takes a planned waypoint trajectory as input and returns a corrected, dynamically feasible rollout by (i) tracking the path to obtain $\delta_{\mathrm{nom},k}$, (ii) projecting the nominal speed onto the CBF-admissible set, and (iii) rolling out the bicycle dynamics with the corrected longitudinal command.


Given a planned trajectory $\hat{\boldsymbol{\tau}}_0^{(t)}$ and predicted neighbor trajectories $\{\hat{\boldsymbol{\tau}}^{(j)}\}_{j \in \mathcal{R}}$, PC-CBF proceeds sequentially over rollout steps $k=1,\ldots,K$. At each rollout step, an LQR tracker produces a nominal control $(a_{\mathrm{nom},k},\delta_{\mathrm{nom},k})$ to follow the remaining waypoints under the bicycle model~\eqref{eq:bicycle}. To preserve path geometry, we fix steering to $\delta_k=\delta_{\mathrm{nom},k}$ and enforce the capsule barrier condition against agents in $\mathcal{R}$ by modulating speed.

Under fixed steering, the capsule barrier time derivative along the bicycle dynamics~\eqref{eq:bicycle} admits the velocity-linear form
\begin{equation}
    \dot{h}^j
    = \underbrace{\frac{\partial h^j}{\partial x}\cos\theta
    + \frac{\partial h^j}{\partial y}\sin\theta
    + \frac{\partial h^j}{\partial \theta}\frac{\tan\delta_{\mathrm{nom},k}}{L}}_{\displaystyle \triangleq\; \frac{\partial h^j}{\partial v}}
    \cdot v,
    \label{eq:hdot_v}
\end{equation}
which yields a velocity-level CBF constraint $\dot{h}^j \geq -\alpha(h^j)$.

\begin{definition}[Fixed-Steering CBF]
\label{def:vel_cbf}
Fix $\delta=\delta_{\mathrm{nom},k}$ and consider the induced rollout dynamics obtained from~\eqref{eq:bicycle} with steering held fixed, where the effective decision variable is the speed $v\ge 0$. Define the safe set
\begin{equation}
\mathcal{C}^j \;=\; \{\mathbf{x} \mid h^j(\mathbf{x}) \ge 0\}.
\end{equation}
We say that the capsule barrier $h^j$ is a \emph{control barrier function} (CBF) for the induced fixed-steering rollout dynamics if $h^j$ is continuously differentiable on $\mathcal{C}^j$ and, for all $\mathbf{x}\in\mathcal{C}^j$, the admissible set
\begin{equation}
K_{\mathrm{cbf}}(\mathbf{x})
=
\Bigl\{v\geq 0 \ \Bigm|\  \dot{h}^j(\mathbf{x}) \ge -\alpha\!\bigl(h^j(\mathbf{x})\bigr)\Bigr\}
\end{equation}
is nonempty.
\end{definition}

\begin{theorem}[Feasibility of Velocity-Level Capsule CBF]
\label{thm:cbf}
Under fixed steering $\delta=\delta_{\mathrm{nom},k}$, for every $\mathbf{x}$ with $h^j(\mathbf{x})\geq 0$, the set
\begin{equation}
K_{\mathrm{cbf}}(\mathbf{x})
=
\Bigl\{v\geq 0 \ \Bigm|\  \frac{\partial h^j}{\partial v}(\mathbf{x})\,v
\geq -\alpha\!\bigl(h^j(\mathbf{x})\bigr)\Bigr\}
\end{equation}
is nonempty.
\end{theorem}

\begin{proof}
By Proposition~\ref{prop:smooth}, $h^j$ is continuously differentiable on the domain where the closest-point pair is unique, hence $\dot{h}^j$ is well-defined. Under fixed steering, $\dot{h}^j = (\partial h^j/\partial v)\,v$ via~\eqref{eq:hdot_v}. Two cases arise:
\emph{(i)} If $\partial h^j/\partial v \geq 0$, any $v\geq 0$ satisfies the constraint.
\emph{(ii)} If $\partial h^j/\partial v < 0$, choosing $v=0$ yields $\dot{h}^j=0\geq -\alpha(h^j)$ since $\alpha(h^j)\geq 0$ for $h^j\geq 0$.
Thus $K_{\mathrm{cbf}}(\mathbf{x})$ is nonempty.
\end{proof}

Proposition~\ref{prop:smooth} and Theorem~\ref{thm:cbf} together establish that the capsule barrier $h^j$ in~\eqref{eq:capsule_cbf} is a CBF for the induced fixed-steering rollout dynamics in the sense of Definition~\ref{def:vel_cbf}.

\begin{corollary}[Forward Invariance]
\label{cor:invariance}
If $v(\mathbf{x})$ is selected such that $v(\mathbf{x}) \in K_{\mathrm{cbf}}(\mathbf{x})$ for all $\mathbf{x}\in\mathcal{C}^j$ (e.g., by solving~\eqref{eq:v_cbf}), then $\mathcal{C}^j$ is forward invariant under the resulting closed-loop rollout dynamics (Section~III-B).
\end{corollary}

We therefore compute the minimally modified safe speed by solving
\begin{equation}
\begin{aligned}
    v_k^* = \arg\min_{v_k}\;& \|v_k - v_{\mathrm{nom},k}\|^2 \\
    \mathrm{s.t.}\quad& \frac{\partial h^j}{\partial v}\, v_k \geq -\alpha(h^j),\quad \forall\, j \in \mathcal{R},
\end{aligned}
    \label{eq:v_cbf}
\end{equation}
and recover the corresponding acceleration as $a_k^* = (v_k^* - v_k)/\Delta t$. This velocity-level implementation avoids a higher-order CBF (HOCBF)~\cite{xiao2021hocbf}, which would require differentiating $\dot{h}^j$ again and introducing additional class-$\mathcal{K}$ hyperparameters, often resulting in more sensitive behavior.

Finally, we propagate the ego state forward using the bicycle model~\eqref{eq:bicycle} with $(a_k^*,\delta_{\mathrm{nom},k})$ and repeat for $k=1,\ldots,K$. The resulting rollout $\hat{\boldsymbol{\tau}}_0^{*} = (\hat{\mathbf{x}}_1^*, \ldots, \hat{\mathbf{x}}_K^*)$ is (i) dynamically feasible by construction, (ii) safe with respect to all agents in $\mathcal{R}$ by enforcing the CBF constraint at every rollout step, and (iii) path-consistent since the spatial path is tracked through $\delta_{\mathrm{nom},k}$ while safety is achieved primarily through speed modulation. The full PC-CBF procedure is summarized in Algorithm~\ref{alg:pc_cbf}.

\begin{algorithm}[ht]
\caption{PC-CBF: Path-Consistent Capsule CBF Safety Filter}\label{alg:pc_cbf}\footnotesize
\KwIn{Planned trajectory $\hat{\boldsymbol{\tau}}_0^{(t)} = (\hat{\mathbf{x}}_1, \ldots, \hat{\mathbf{x}}_K)$, neighbor trajectories $\{\hat{\boldsymbol{\tau}}^j\}_{j \in \mathcal{R}}$}
\KwOut{Corrected trajectory $\hat{\boldsymbol{\tau}}_0^{*}$}
$\hat{\mathbf{x}}_1^* \leftarrow \hat{\mathbf{x}}_1$\;
\For{$k = 1,\, \ldots,\, K{-}1$}{
    $(a_{\mathrm{nom},k},\, \delta_{\mathrm{nom},k}) \leftarrow \mathrm{LQR}(\hat{\mathbf{x}}_k^*,\, (\hat{\mathbf{x}}_{k+1}, \ldots, \hat{\mathbf{x}}_K))$ \tcp*{track trajectory}
    $v_{\mathrm{nom},k} \leftarrow v_k + a_{\mathrm{nom},k}\,\Delta t$\;
    \For{$j \in \mathcal{R}$}{
        Compute $\frac{\partial h^j}{\partial v}$ via Eq.~\eqref{eq:hdot_v}\;
    }
    $v_k^* \leftarrow$ solve Eq.~\eqref{eq:v_cbf} \tcp*{safe speed}
    $a_k^* \leftarrow (v_k^* - v_k) / \Delta t$ \tcp*{recover acceleration}
    $\hat{\mathbf{x}}_{k+1}^* \leftarrow \hat{\mathbf{x}}_k^* + \Delta t\, f(\hat{\mathbf{x}}_k^*,\, a_k^*,\, \delta_{\mathrm{nom},k})$ \tcp*{propagate}
}
\Return $\hat{\boldsymbol{\tau}}_0^{*} = (\hat{\mathbf{x}}_1^*,\, \ldots,\, \hat{\mathbf{x}}_K^*)$\;
\end{algorithm}

\subsection{Integrating Certifiable Structure into Diffusion Planning}

Having established \emph{what} certifiable structure to incorporate (Q1), we now address \emph{how} to integrate it into the diffusion planner (Q2). A natural baseline is to apply PC-CBF as a post-hoc filter on the final denoised trajectory $\boldsymbol{\tau}_0$. 
However, such a one-shot correction is oblivious to the diffusion model’s internal generation: the planner has no opportunity to adapt, and the corrected output may be unlikely under the learned distribution.

Our key insight is to enforce the certifiable structure \emph{within} the denoising process by operating on the predicted clean trajectory estimate. At each denoising step $t$, the diffusion model produces a noisy trajectory $\boldsymbol{\tau}_t$, from which we extract the predicted clean trajectory $\hat{\boldsymbol{\tau}}_0^{(t)}$ via~\eqref{eq:clean_estimate}. This clean estimate is the model’s current best estimate of the final plan and is therefore a meaningful object on which to enforce rollout-time safety, unlike the noisy state $\boldsymbol{\tau}_t$ itself.

We apply PC-CBF to $\hat{\boldsymbol{\tau}}_0^{(t)}$ to obtain the corrected estimate $\hat{\boldsymbol{\tau}}_0^{*(t)}$, and then re-noise it back into the diffusion process:
\begin{equation}
    \boldsymbol{\tau}_{t-1} = \sqrt{\bar{\alpha}_{t-1}}\;\hat{\boldsymbol{\tau}}_0^{*(t)} + \sqrt{1 - \bar{\alpha}_{t-1}}\;\hat{\boldsymbol{\epsilon}}^{(t)} + \sigma_t \mathbf{z},
    \label{eq:corrected_denoise}
\end{equation}
where $\hat{\boldsymbol{\epsilon}}^{(t)} = \bigl(\boldsymbol{\tau}_t - \sqrt{\bar{\alpha}_t}\,\hat{\boldsymbol{\tau}}_0^{*(t)}\bigr) / \sqrt{1 - \bar{\alpha}_t}$ is the re-estimated noise consistent with the corrected clean estimate and $\mathbf{z} \sim \mathcal{N}(\mathbf{0}, \mathbf{I})$. Subsequent denoising steps then operate on a trajectory that already accounts for safety, enabling \textbf{co-adaptation}: the diffusion planner refines the corrected estimate back toward its learned distribution while preserving the certifiable structure. As $t \to 0$, $\hat{\boldsymbol{\tau}}_0^{(t)}$ sharpens and the required safety corrections diminish, converging to a trajectory that is both safe and consistent with the learned behavior.

\subsection{PC-Diffuser and Practical Considerations}
\label{sec:practical}

Enforcing the CBF constraint against all $N$ neighboring agents is often unnecessarily conservative and computationally expensive, since most nearby agents pose no collision risk (e.g., vehicles traveling in the opposite direction on separated lanes or parked vehicles on the side). Instead of relying on hand-crafted geometric heuristics, we leverage the diffusion planner's predicted neighbor trajectories to focus certification on a smaller set of safety-critical agents.

Concretely, at each denoising step $t$ we evaluate the minimum barrier value between the ego's predicted trajectory and each neighbor $j$:
\begin{equation}
    h^j_{\min} = \min_{k \in \{1,\ldots,K\}} h^j\!\left(\hat{\mathbf{x}}_k,\;\hat{\mathbf{x}}_k^j\right).
\end{equation}
Agent $j$ is added to the safety-critical set $\mathcal{R}$ whenever $h^j_{\min} \leq \eta$ for a user-specified threshold $\eta \geq 0$. The set $\mathcal{R}$ is accumulated across denoising steps: once an agent is flagged as critical at step $t$, it remains in $\mathcal{R}$ for all subsequent steps. This monotonic accumulation prevents oscillatory inclusion and exclusion as $\hat{\boldsymbol{\tau}}_0^{(t)}$ sharpens during denoising. 

Our full framework, summarized in Algorithm~\ref{alg:pc_cbf_denoise}, proceeds as follows at each denoising step $t$:
\begin{enumerate}
    \item \textbf{Predict.} Extract the clean trajectory estimate $\hat{\boldsymbol{\tau}}_0^{(t)}$ from the noisy state via~\eqref{eq:clean_estimate}.
    \item \textbf{Filter.} Update the safety-critical agent set $\mathcal{R}$ based on predicted proximity.
    \item \textbf{Correct.} Apply PC-CBF (Section~\ref{sec:pc_cbf}) to obtain the safe, dynamically feasible, path-consistent trajectory $\hat{\boldsymbol{\tau}}_0^{*(t)}$.
    \item \textbf{Re-noise.} Inject the corrected estimate back into the diffusion process via~\eqref{eq:corrected_denoise}.
\end{enumerate}
After all denoising steps complete, the first action of $\boldsymbol{\tau}_0^*$ is executed and the planning loop resets.

\begin{algorithm}[ht] \footnotesize
\caption{PC-Diffuser: Denoising with Path-Consistent CBF}\label{alg:pc_cbf_denoise}
\KwIn{Joint noisy trajectory $\boldsymbol{\tau}_T$ including ego and neighbors}
\KwOut{Safe denoised trajectory $\boldsymbol{\tau}_0^*$}
$\mathcal{R} \leftarrow \emptyset$\;
\For{$t = T,\, T{-}1,\, \ldots,\, 1$}{
    $\hat{\boldsymbol{\tau}}_0^{(t)} \leftarrow \text{Eq.~\eqref{eq:clean_estimate}}$ \tcp*{predict clean estimate}
    $\mathcal{R} \leftarrow \mathcal{R} \cup \mathrm{ProximityFilter}(\hat{\boldsymbol{\tau}}_0^{(t)})$ \tcp*{accumulate critical agents}
    $\hat{\boldsymbol{\tau}}_0^* \leftarrow \text{PC-CBF}(\hat{\boldsymbol{\tau}}_0^{(t)},\, \{\hat{\boldsymbol{\tau}}^j\}_{j \in \mathcal{R}})$ \tcp*{Alg.~\ref{alg:pc_cbf}}
    $\hat{\boldsymbol{\epsilon}} \leftarrow \bigl(\boldsymbol{\tau}_t - \sqrt{\bar{\alpha}_t}\,\hat{\boldsymbol{\tau}}_0^*\bigr) / \sqrt{1 - \bar{\alpha}_t}$ \tcp*{re-estimate noise}
    $\boldsymbol{\tau}_{t-1} \leftarrow \sqrt{\bar{\alpha}_{t-1}}\,\hat{\boldsymbol{\tau}}_0^* + \sqrt{1-\bar{\alpha}_{t-1}}\,\hat{\boldsymbol{\epsilon}} + \sigma_t \mathbf{z}$ \tcp*{diffusion step}
}
\Return $\boldsymbol{\tau}_0^*$\;
\end{algorithm}

\section{Evaluation}

We aim to validate two hypotheses: (H1: the structure to enforce)~jointly enforcing safety, dynamic feasibility, and path-consistency reduces collisions while maintaining driving quality, and (H2: the way to integrate)~integrating corrections iteratively within denoising, instead of just applying once in a post-hoc manner, keeps the corrected trajectory closer to the learned distribution, yielding safer long-horizon behavior. 
To evaluate these claims, we test on the nuPlan closed-loop benchmark using DiffusionPlanner~\cite{zheng2025diffusionplanner} as our base model, one of the strongest-performing learning-based models that jointly forecast neighbors' trajectories, which we can leverage to construct our certifiable structure. 
For baselines, we compare against popular safety augmentation strategies: Classifier Guidance~\cite{zheng2025diffusionplanner}, diffusion-time barrier constraints~\cite{xiao2023safediffuser}, and optimization-based safety filters~\cite{zeng2021mpccbf}, which are general safety augmentation methods that do not leverage nuPlan's neighbor policy (IDM~\cite{treiber2000idm}) as a prior or simulate multiple proposals on-the-fly to better reflect real-life performance. 

\subsection{Dataset and Metrics}

The nuPlan benchmark~\cite{caesar2021nuplan} provides approximately 1,300 hours of expert driving data collected across four cities (Las Vegas, Boston, Pittsburgh, Singapore) with auto-labeled object tracks and traffic light states. 
Our evaluation uses nuPlan's closed-loop reactive simulation, which represents logged traffic with a bird's-eye-view (BEV) representation and resimulates agent behavior using the Intelligent Driver Model (IDM)~\cite{treiber2000idm} for longitudinal car-following.

We evaluate on two standard splits:
\emph{Val14}, which contains up to 100 scenarios per type across 14 scenario types and \emph{Test14-hard}, a curated 280-scenario subset in which rule-defined planners struggle.

To further isolate safety-critical performance, we extract \emph{all} 68 scenarios from Val14 in which the base model DiffusionPlanner~\cite{zheng2025diffusionplanner} collides (100\% collision rate by construction), representing roughly 7\% of Val14 and constituting its most challenging subset. We refer to this subset as the \emph{all-collision challenge set}.


We report two metrics: the \emph{collision rate} and nuPlan's built-in \emph{composite score}~\cite{caesar2021nuplan}.
\begin{equation}
    s = \prod_{i \in \mathcal{M}} m_i \;\cdot\; \sum_{j \in \mathcal{A}} w_j \, a_j,
    \label{eq:nuplan_score}
\end{equation}
where $m_i \in \{0, 0.5, 1\}$ are hard multiplier metrics (at-fault collision, drivable area compliance, progress, driving direction) that zero or halve the score upon violation, and $a_j$ are weighted soft metrics including time-to-collision maintenance ($w=5$), route progress ($w=5$), speed limit compliance ($w=4$), and comfort ($w=2$).

\subsection{Quantitative Analysis Against Baselines}

\begin{table}[ht]
    \centering
    \caption{Closed-loop results on the all-collision challenge set. Collision rate (\%) and nuPlan composite score (0--1) are reported.} 
    \label{tab:collision_subset}
    \adjustbox{max width=\columnwidth}{
        \begin{tabular}{l|c|c|c|c|c}
        \hline
            Method & \makecell{Collision\\Rate} & \makecell{Composite\\Score} & Safety & \makecell{Dynamic\\Feasibility} & \makecell{Path\\Consistency}\\
        \hline
        Vanilla (DiffusionPlanner) & 100\%          & 0.00            & \xmark & \xmark & \xmark\\
        Classifier Guidance        & 74.29\%        & 0.15            & \xmark & \xmark & \xmark \\
        Robust-Safe Diffuser       & 81.43\%        & 0.10            & \cmark & \xmark & \xmark \\
        Relaxed-Safe Diffuser      & 78.57\%        & 0.15            & \cmark & \xmark & \xmark \\
        Time-Varying-Safe Diffuser & 88.57\%        & 0.07            & \cmark & \xmark & \xmark \\
        MPC-CBF                    & 79.29\%        & 0.10            & \cmark & \cmark & \xmark \\
        PC-Diffuser (ours)         & \textbf{10.29}\% & \textbf{0.59} & \cmark & \cmark & \cmark \\
    \end{tabular}
    }
\end{table}


We compare PC-Diffuser against three safety augmentation families applied to the same base planner: Classifier Guidance, SafeDiffuser, and MPC-CBF. Classifier Guidance~\cite{zheng2025diffusionplanner} steers the denoising process via gradients of a hand-crafted distance classifier. SafeDiffuser~\cite{xiao2023safediffuser} incorporates barrier constraints on the denoising process with three barrier-value scheduling variants (Robust-Safe Diffuser, Relaxed-Safe Diffuser, Time-Varying-Safe Diffuser) to mitigate a \emph{local trap problem} (denoising process cannot progress without violating CBF constraint). Lastly, MPC-CBF~\cite{zeng2021mpccbf} embeds CBF constraints into the receding-horizon optimization, aiming to satisfy the CBF constraint throughout the receding horizon. PC-Diffuser reduces the collision rate from 100\% to 10.29\% on the all-collision challenge set, compared to 74--89\% for all baselines, while achieving the highest composite score (0.59 vs.\ $\leq$0.15).



We further evaluate on the full Val14 and Test14-hard splits to verify that safety does not degrade overall driving quality. Despite the additional constraints from our certifiable structure, PC-Diffuser improves the base model's composite score from \emph{0.83} to \emph{0.88} on Val14 and from \emph{0.69} to \emph{0.78} on Test14-hard, indicating that the certifiable structure enhances safety without sacrificing driving quality.


The significant improvement in both collision rate and composite score on the all-collision challenge set (Table~\ref{tab:collision_subset}) while preserving strong driving performance on full data splits demonstrates the effectiveness of our certifiable structure and PC-Diffuser framework supporting H1 and H2.



\begin{figure*}[ht]
    \centering
    \begin{subfigure}[t]{0.19\textwidth}
        \centering
        \includegraphics[trim=150 225 225 225, clip, width=\textwidth]{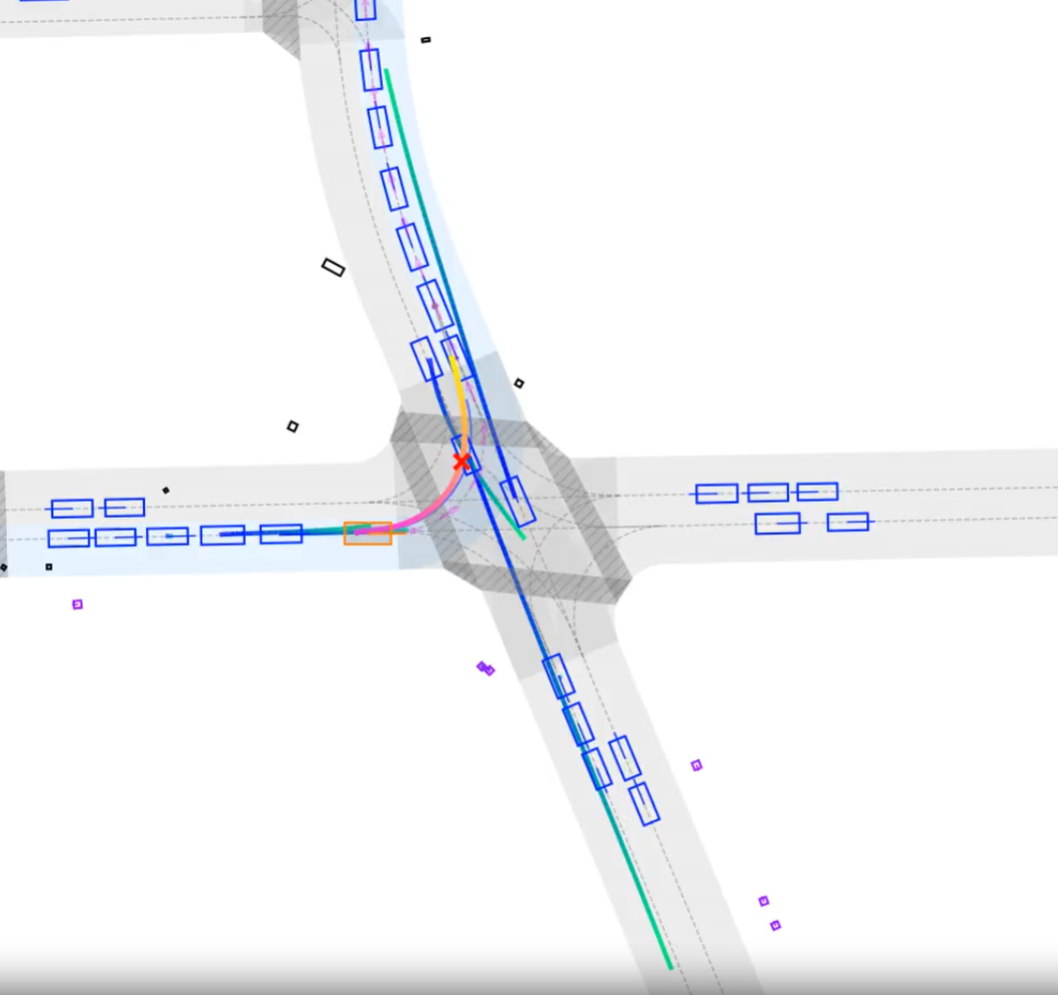}
        \caption{Diffuser (vanilla)}
        \label{fig:qual_vanilla}
    \end{subfigure}
    \hfill
    \begin{subfigure}[t]{0.19\textwidth}
        \centering
        \includegraphics[trim=150 225 225 225, clip, width=\textwidth]{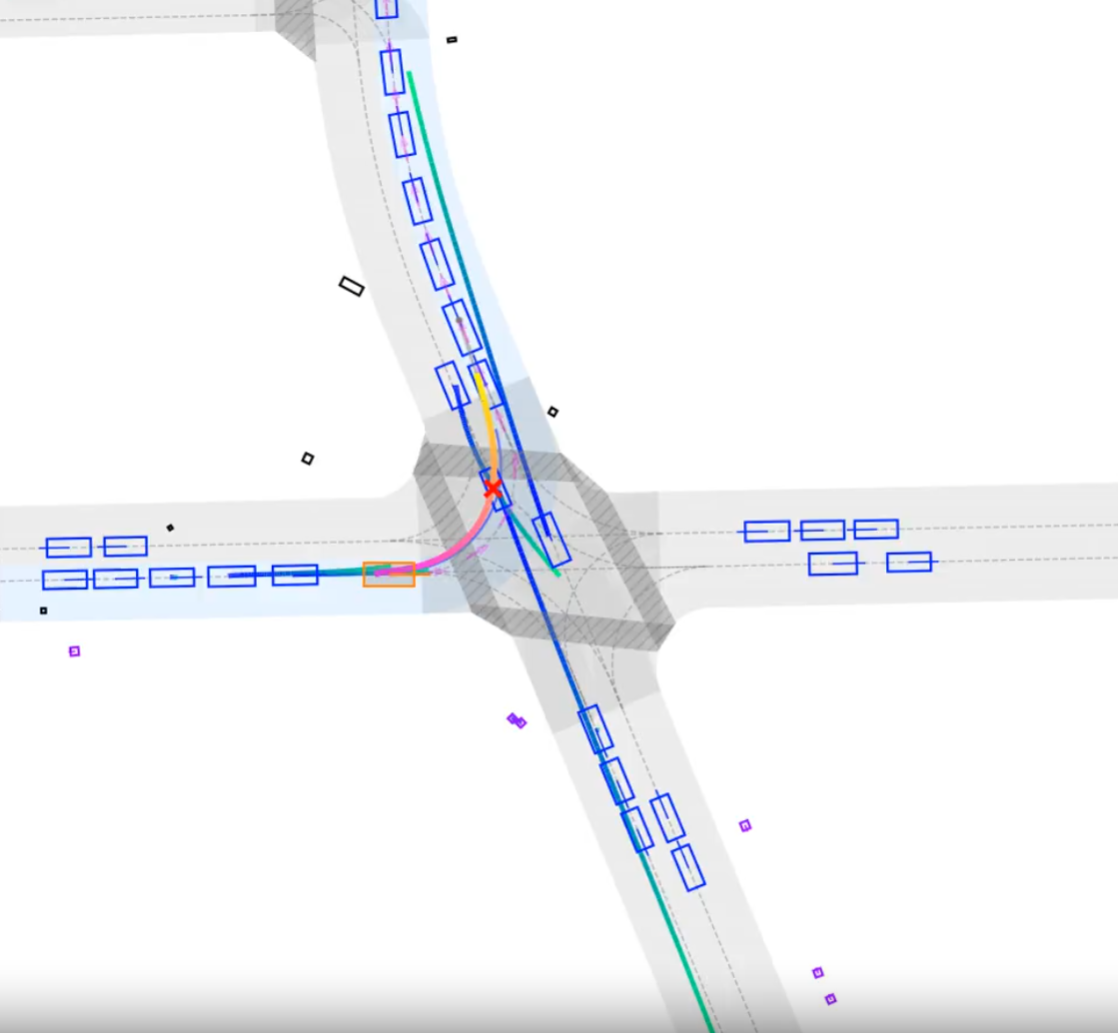}
        \caption{Classifier Guidance}
        \label{fig:qual_classifier}
    \end{subfigure}
    \hfill
    \begin{subfigure}[t]{0.19\textwidth}
        \centering
        \includegraphics[trim=200 275 275 275, clip, width=\textwidth]{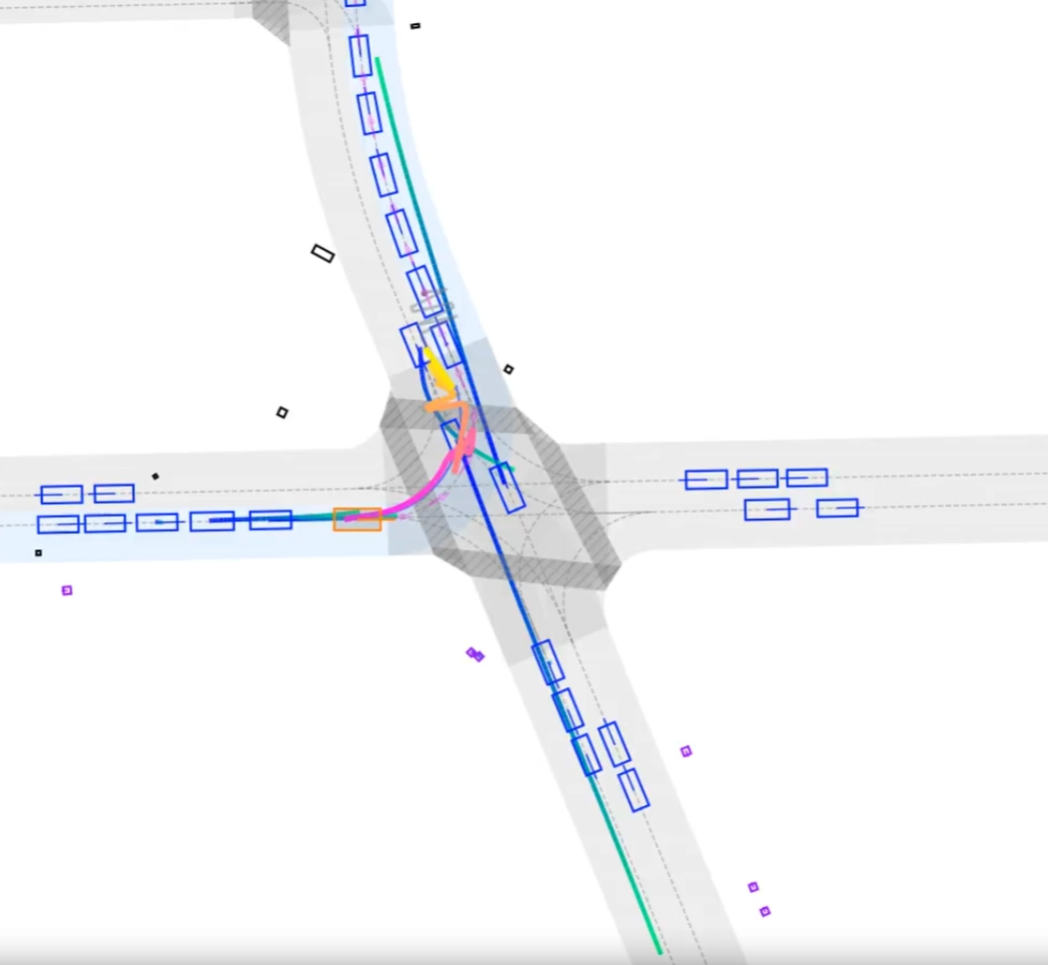}
        \caption{Relaxed-Safe Diffuser}
        \label{fig:qual_safediffuser}
    \end{subfigure}
    \hfill
    \begin{subfigure}[t]{0.19\textwidth}
        \centering
        \includegraphics[trim=150 225 225 225, clip, width=\textwidth]{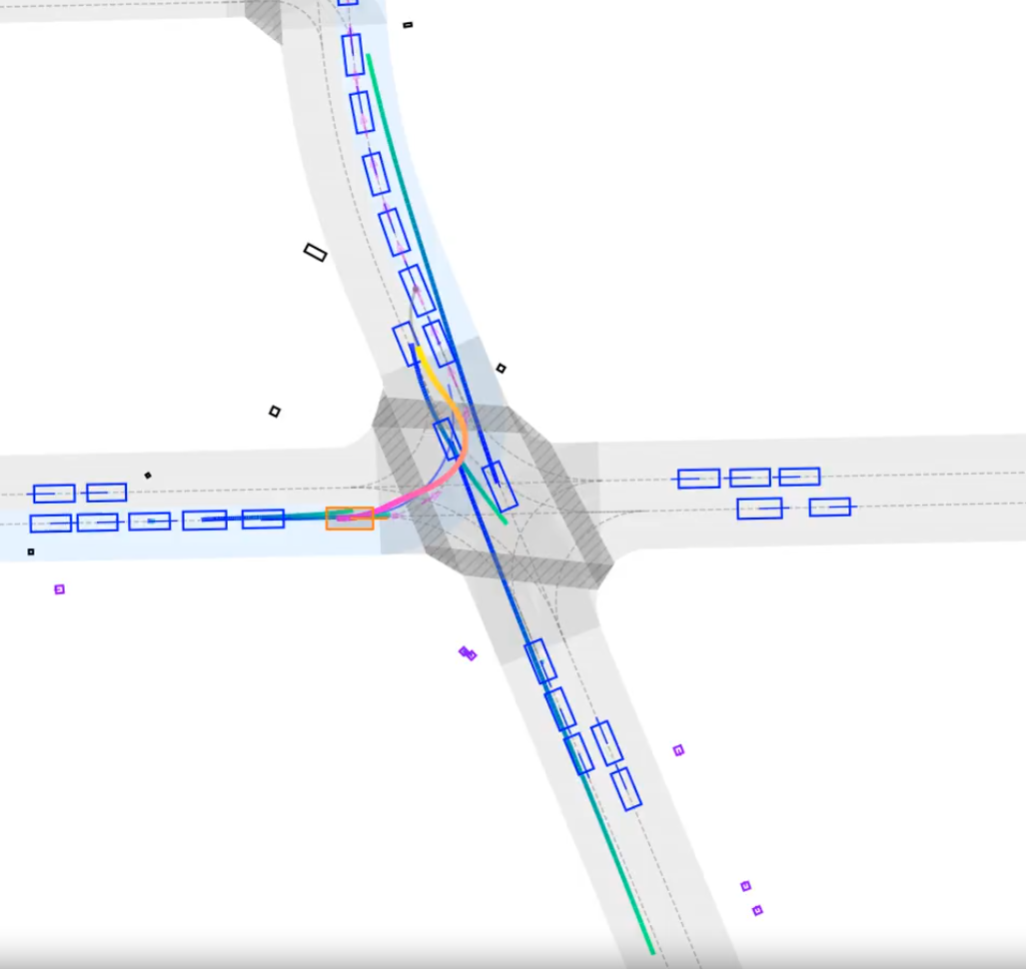}
        \caption{MPC-CBF}
        \label{fig:qual_mpc}
    \end{subfigure}
    \hfill
    \begin{subfigure}[t]{0.19\textwidth}
        \centering
        \includegraphics[trim=150 225 225 225, clip, width=\textwidth]{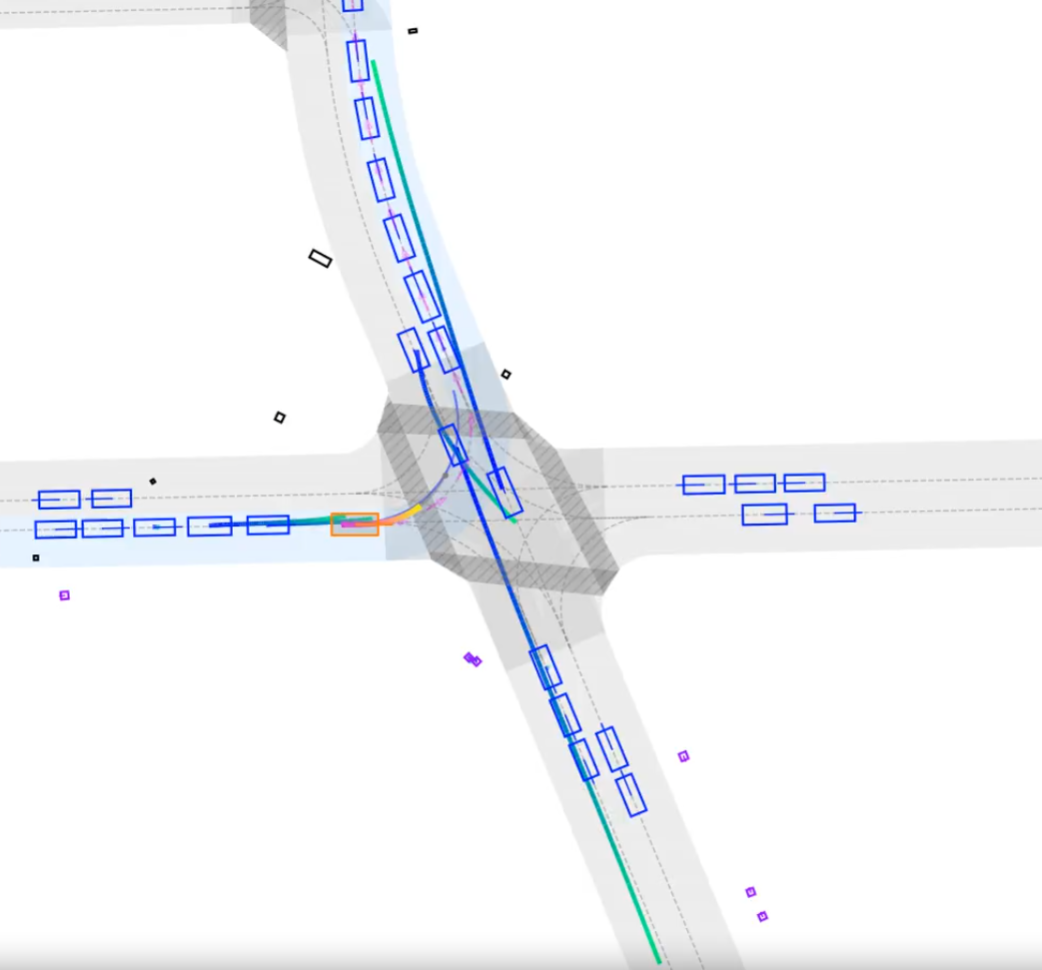}
        \caption{PC-Diffuser (ours)}
        \label{fig:qual_ours}
    \end{subfigure}
    \caption{Qualitative comparison on a collision critical intersection scenario. The ego vehicle (orange box) approaches from the left. Planned trajectory for ego and neighbors are depicted in bright gradient and blue-green gradient, respectively. The first collision is marked with a red x. The vanilla Diffuser~(a) and Classifier Guidance~(b) collide with oncoming traffic. SafeDiffuser~(c) produces a safe but dynamically infeasible trajectory. MPC-CBF~(d) produces a safe and feasible trajectory but deviates into the oncoming lane. PC-Diffuser~(e) yields to the oncoming traffic before safely making a left turn while preserving a lane-consistent trajectory. More qualitative comparisons are available on the project website.}
    \label{fig:qualitative}
\end{figure*}

\subsection{Qualitative Analysis Against Baselines}
Figure~\ref{fig:qualitative} provides a qualitative comparison on a safety-critical intersection scenario where the ego vehicle has to make a left turn into a congested lane while avoiding two oncoming vehicles. In this scenario, due to the oncoming traffic and high congestion, the ideal maneuver is to yield and slowly make a left turn when available.

As shown in Fig. \ref{fig:qualitative}(a), vanilla DiffusionPlanner~\cite{zheng2025diffusionplanner}, trained to predict the most likely trajectory based on its training data, produces a natural and smooth left-turn trajectory. 
However, without a safety certificate, it fails to avoid a collision with the oncoming vehicle. Similarly, Classifier Guidance~\cite{zheng2025diffusionplanner} (via distance energy-based classifier), as shown in Fig. \ref{fig:qualitative}(b),  also produces a smooth, high-quality trajectory but without the safety certificate, it also fails to avoid collision. On the other hand, SafeDiffuser, as shown in Fig. \ref{fig:qualitative}(c), provides a safety certificate but without the dynamic feasibility, it produces a physically meaningless trajectory as the denoising process was interfered by CBF constraint on the diffusion time. The Relaxed-Safe Diffuser variant was introduced to tackle this problem; however, it is not foolproof and still produces a physically infeasible trajectory. MPC-CBF, as shown in Fig. \ref{fig:qualitative}(d), provides both safety certificate and dynamic feasibility through CBF and model dynamics constraints. Yet, as it is geometrically unconstrained, it ignores the intended behavior of the diffusion planner and swerves into the oncoming lane, which is highly undesirable. In contrast to these baseline methods, PC-Diffuser, as shown in Fig. \ref{fig:qualitative}(e), simultaneously satisfies safety, dynamic feasibility, and path-consistency, and yields to the oncoming traffic before initiating a left turn while remaining lane-consistent.


The qualitative failure analysis of the baseline methods and the safe maneuver of PC-Diffuser at the safety-critical intersection demonstrates the importance of jointly enforcing safety, dynamic feasibility, and path-consistency, supporting H1.



\begin{table}[!t]
    \centering
    \caption{Ablation study on the all-collision challenge set. Each row excludes one component from PC-Diffuser. None (PC-Diffuser) shows the same score as in Table~\ref{tab:collision_subset}.}
    \label{tab:ablation}
    \adjustbox{max width=\columnwidth}{
    \begin{tabular}{l|c|c}
        \hline
        Excluded Component              & Collision Rate & Composite Score \\
        \hline
        Iterative safeguard             & 16.91\%  & 0.47 \\
        Selective filtering             & 11.76\%  & 0.47 \\
        Dynamic feasibility             & 21.32\%  & 0.30 \\
        None (full PC-Diffuser)         & \textbf{10.29}\% & \textbf{0.59} \\
    \end{tabular}
    }
\end{table}

\subsection{Ablation Study}

To understand the contribution of each component of PC-Diffuser, we systematically remove one component at a time and re-evaluate on the all-collision challenge set (Table~\ref{tab:ablation}).

\begin{figure}[!t]
    \centering
    \includegraphics[width=\columnwidth]{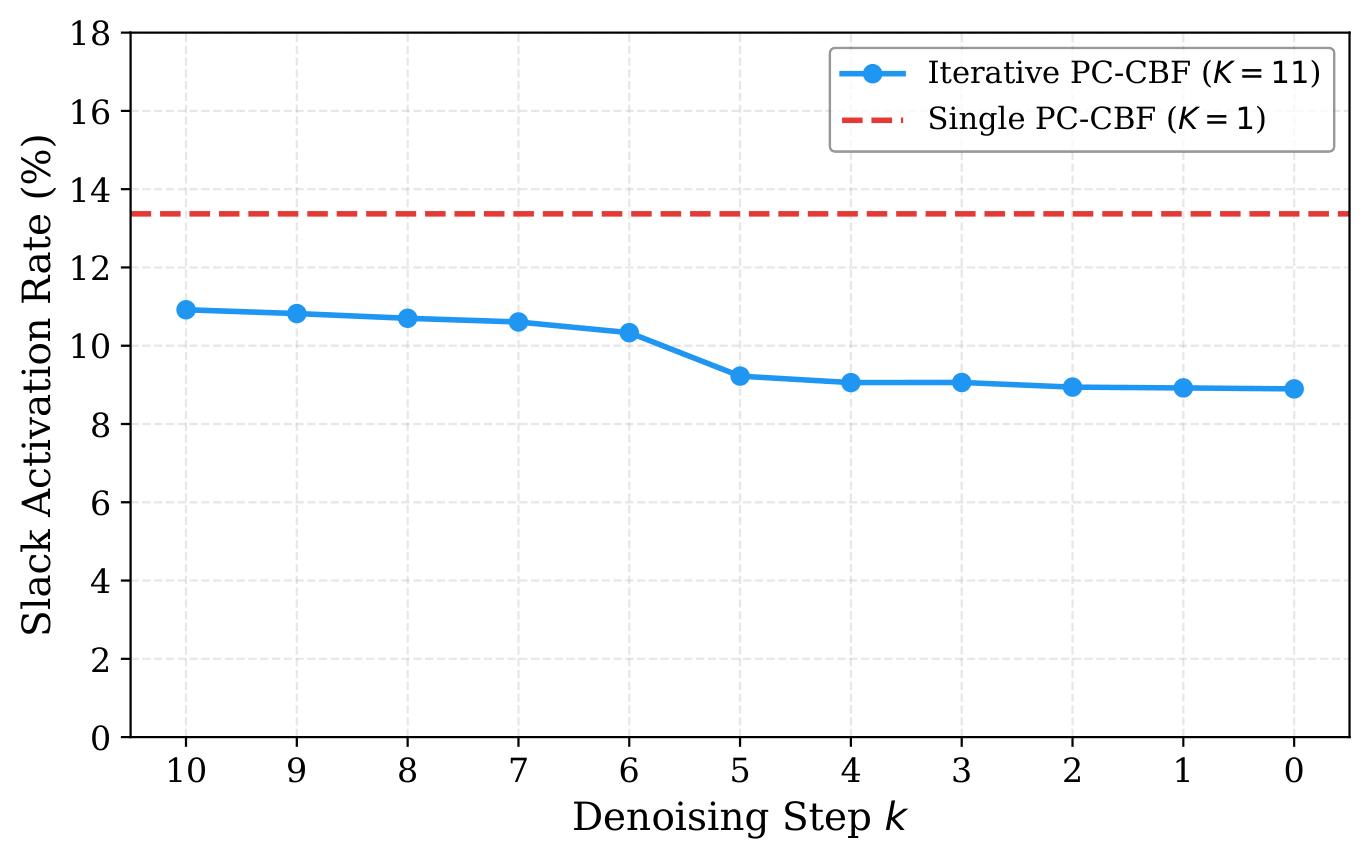}
    \caption{Average slack activation rate (\%) in the CBF-QP across denoising steps. Iterative PC-CBF shows monotonically decreasing corrections as denoising progresses, indicating convergence to a safe trajectory. Single-step PC-CBF (dashed line) applies a larger correction at the final step with higher slack activation rate due to the absence of iterative refinement.}
    \label{fig:slack_counts}
\end{figure}


\textbf{Iterative safeguard.} Applying PC-CBF only at the final denoising step, rather than throughout the process, increases the collision rate by ${\sim}6\%$ and lowers the composite score by 0.12. To understand why, we examine how often the slack variables in the CBF-QP are activated across denoising steps (Figure~\ref{fig:slack_counts}). With iterative correction, slack activation decreases monotonically as denoising progresses, producing trajectories that increasingly satisfy safety constraints on their own. By contrast, single-step correction sees no such convergence and must apply a larger one-shot adjustment at the end. The benefit compounds over simulation time: the gap between slack activation during the first denoising step and post-hoc correction is $\sim$2.5\% (Figure~\ref{fig:slack_counts}) while when we only consider slack violation rate of the first simulation step, the gap is much smaller at 0.65\%. An improvement in the slack activation rate over long-term indicates that iterative correction steers the diffusion model toward trajectories that yield safer long-horizon behavior by jointly searching for a safe trajectory together with diffusion planner, supporting H2. Regarding latency, a single post-hoc fix incurs 1.5x the vanilla latency whereas iterative safeguard incurs 5x (2 fps).

\textbf{Selective filtering.} Not filtering out benign agents degrades the composite score by 0.12 as expected, since our capsule CBF constraints can hinder the vehicle's progress by promoting unnecessarily conservative behavior (e.g. vehicles in the oncoming lane). Interestingly removing the filtering provides a modest improvement on collision rate ($+1.5\%$). This is mainly due to PC-Diffuser occasionally finding a safe-set which without the selective filtering would've been considered unsafe due to overly conservative constraint.

Among all components, \textbf{Dynamic feasibility} has the largest impact. To isolate this component while maintaining the rest of PC-Diffuser, we replace the LQR path-following controller with an arc-reparameterization that still satisfies CBF constraints and path-consistency. Removing dynamic feasibility raises the collision rate by ${\sim}11\%$ and lowers the composite score by 0.29, indicating that kinematic grounding is crucial for both safety and driving quality, which further supports H1. We note that despite the lack of dynamic feasibility, the arc-reparameterization method outperforms all baseline methods (Table~\ref{tab:collision_subset}) as path-consistency still respects the learned dynamics from the diffusion model, mitigating much of the adverse effects of kinematic infeasibility.

\section{Conclusion \& Limitations}

We presented PC-Diffuser, a safety augmentation framework that integrates a certifiable structure into diffusion-based planning through path-consistent barrier functions.
We introduced a novel capsule distance CBF safety filter, a framework to enforce dynamic feasibility on trajectory-space diffusion models, a way to reduce deviation from learned distribution via geometric consistency, and an innovative way to integrate these properties during the denoising process (generation) rather than as a post-hoc correction.

Despite these efforts, our approach has several limitations. 
First, although the iterative integration of PC-CBF allows for neighbors' trajectory forecast to adapt according to PC-Diffuser's correction on ego's plan, our framework still remains reactive and is limited in anticipating reactions from human drivers, which could potentially be improved via active uncertainty mitigation methods~\cite{wang2023active,yiwei2026icra}. Additionally, probabilistic CBF ~\cite{lyu2021probabilistic} and risk measures such as CVaR~\cite{lyu2023risk} could improve robustness to uncertainty in predicted future states of neighboring vehicles.

Second, our evaluation relies on nuPlan's IDM-based reactive simulation, where IDM agents can sometimes unnaturally swarm the ego vehicle from all directions, causing unavoidable collisions and yielding a $\sim$10\% collision rate with PC-Diffuser on the all-collision challenge set.
Benchmarking on a higher-fidelity simulator with naturalistic neighbor vehicle policy is left to future work but constitutes a substantial research challenge in its own right.

\addtolength{\textheight}{-12cm}   








\bibliographystyle{ieeetr}
\bibliography{references}

\end{document}